\documentclass[lettersize,journal]{IEEEtran}
\usepackage{amsmath,amsfonts}
\usepackage{adjustbox,bm}
\usepackage{amssymb}
\usepackage{amsthm} 
\usepackage{latexsym}
\usepackage{amsbsy}
\usepackage{mathtools}
\usepackage{algorithmic}
\usepackage{algorithm}
\usepackage{array}
\usepackage{textcomp}
\usepackage{stfloats}
\usepackage{url}
\usepackage{verbatim}
\usepackage{graphicx}
\usepackage{cite}

\usepackage{hyperref}
\usepackage{booktabs,multicol,multirow}
\usepackage{enumerate}
\usepackage[inline]{enumitem}
\usepackage{color}
\usepackage{float,caption}
\usepackage{tikz}

  \newcommand{\beq}{\begin{equation}}
    \newcommand{\eeq}{\end{equation}}
    
    \newcommand{\bal}{\begin{align}}
    \newcommand{\eal}{\end{align}}
    \newcommand{\bals}{\begin{align*}}
    \newcommand{\eals}{\end{align*}}
    
\newcommand{\XCal}{\mathcal{X}}
\newcommand{\YCal}{\mathcal{Y}}

\newcommand{\LL}{\mathcal{L}}

\newcommand{\E}{\mathbb{E}}

\newcommand{\R}{\mathbb{R}}

\newcommand{\N}{\mathbb N}

\DeclareMathOperator\argmin{\arg \min}

\newtheorem{theorem}{Theorem}[section]
\newtheorem{lemma}[theorem]{Lemma}
\newtheorem{proposition}[theorem]{Proposition}

\newtheorem{assump}[theorem]{Assumption}

\numberwithin{equation}{section}
\makeatother

\makeatletter
\def\dasharrowfill@#1#2#3#4{%
        $\m@th
        \thickmuskip0mu
        \medmuskip\thickmuskip
        \thinmuskip\thickmuskip
        \relax
        #4#1\mkern2mu
        \xleaders\hbox{$#4\mkern2mu#2\mkern2mu$}\hfill
        \mkern2mu
        #3$%
}

\def\dashleftarrowfill@{\dasharrowfill@\leftarrow\relbar\relbar}
\def\dashrightarrowfill@{\dasharrowfill@\relbar\relbar\rightarrow}
\def\dashleftrightarrowfill@{\dasharrowfill@\leftarrow\relbar\rightarrow}
\def\dashLeftarrowfill@{\dasharrowfill@\Leftarrow\Relbar\Relbar}
\def\dashRightarrowfill@{\dasharrowfill@\Relbar\Relbar\Rightarrow}
\def\dashLeftrightarrowfill@{\dasharrowfill@\Leftarrow\Relbar\Rightarrow}

\providecommand*\xdashleftarrow[2][]{%
  \ext@arrow 0055{\dashleftarrowfill@}{#1}{#2}}
\providecommand*\xdashrightarrow[2][]{%
  \ext@arrow 0055{\dashrightarrowfill@}{#1}{#2}}
\providecommand*\xdashleftrightarrow[2][]{%
  \ext@arrow 0055{\dashleftrightarrowfill@}{#1}{#2}}
\providecommand*\xdashLeftarrow[2][]{%
  \ext@arrow 0055{\dashLeftarrowfill@}{#1}{#2}}
\providecommand*\xdashRightarrow[2][]{%
  \ext@arrow 0055{\dashRightarrowfill@}{#1}{#2}}
\providecommand*\xdashLeftrightarrow[2][]{%
  \ext@arrow 0055{\dashLeftrightarrowfill@}{#1}{#2}}
\makeatother

\begin{document}

\title{Sample Complexity of Transfer Learning: An Optimal Transport Approach}

\author{Haoyang Cao, Xin Guo, Wenpin Tang, Guan Wang 
\thanks{Manuscript submitted May 19, 2026.}
\thanks{H. Cao is with Department of Applied Mathematics and Statistics, John Hopkins University, Baltimore, MD, USA.}
\thanks{X. Guo is with Department of Industrial Engineering and Operations Research, University of California, Berkeley, CA, USA.}
\thanks{W. Tang is with Department of Industrial Engineering and Operations Research, Columbia University, New York, NY, USA.}
\thanks{G. Wang is with Tsinghua-Berkeley Shenzhen Institute, Shenzhen, Guangdong, China.}
}

\markboth{
}%
{Cao \MakeLowercase{\textit{et al.}}: Sample Complexity of Transfer Learning: An Optimal Transport Approach}


\maketitle

\begin{abstract}
Transfer learning is an essential technique for many machine learning/AI models of complex structures such as large language models and generative AI. The essence of transfer learning is to leverage knowledge from resolved source tasks for a new target task, especially when the sample size $m$ of the training data for the latter is low. In this work, we rigorously analyze the potential benefit of  transfer learning in terms of sample efficiency. Specifically, taking  an optimal transport viewpoint of transfer learning, we find that when the data dimension $d$ is higher than $3$, the sample complexity for transfer learning is $O(m^{-(\alpha+1)/d})$, with $\alpha$ indicating the smoothness of the data distribution, as opposed to the $O(m^{-p/d})$ sample complexity for direct learning with $p$ indicating the smoothness of the optimal target model. Our finding theoretically supports a better sample efficiency for transfer learning, when the target task is optimizing over a family of not-so-smooth models (i.e., highly complex networks with the possible use of non-smooth activation functions). Using image classification as an example, we numerically demonstrate the sample efficiency for transfer learning, that is, in the data hungry regime, the model performance can be significantly improved by transfer learning.
\end{abstract}


\section{Introduction}

\IEEEPARstart{T}{ransfer} learning is a machine learning technique where a model developed for a resolved task is incorporated into the training for a model on a new and related task.
The key idea is to leverage knowledge from the resolved learning problem for which the data is abundant to improve the learning outcome in the new learning task where the data are scarce. For example, it has now become a common practice in image classification to reuse models such as \texttt{ResNet-50} pretrained over a large dataset of generic images (e.g., ImageNet). Transfer learning has demonstrated empirical success across different fields and is particularly useful when the target domain has limited labeled data, and when source and target tasks share underlying structure (e.g., image features, language patterns), including medical image analysis \cite{zeng2019automatic, wang2022transfer, kim2022transfer}, natural language processing \cite{ruder2019transfer, devlin-etal-2019-bert, sung2022vl}, large multimodal models \cite{openai2023gpt4,comanici2025gemini,grattafiori2024llama}, and biomedicine \cite{xu2025ibex,apellaniz2025cancer}.


\paragraph{Transfer learning and sample complexity.}

Sample complexity refers to the number of training examples needed for a model to learn a task to a desired level of accuracy. Mathematically, this can be framed using PAC learning with prior knowledge \cite{eisenberg1990sample,hanneke2016optimal}, Bayesian approaches \cite{haussler1994bounds}, or domain adaptation bounds \cite{ben2012hardness}.
The empirical success has demonstrated that transfer learning can reduce sample complexity if the source task helps constrain the hypothesis space for the target task.
In particular, if the source and target tasks are well-aligned, fewer samples are needed to achieve generalization. Theoretically, there has been a line of works analyzing the sample complexity for various transfer learning scenarios. 
In distributional
transfer, the authors of \cite{hanneke2019value} introduce transfer exponents to measure the
target-relevant value of source samples, while \cite{cai2021nonparametric} and
\cite{reeve2021adaptive} derive minimax and adaptive rates for classification
under source--target drift. These works formalize transfer as an effective
sample-size improvement where the source
contribution depends on problem-specific relatedness. Complementary lower-bound
results show that such gains are not automatic: without appropriate
distributional information, multitask or transfer procedures may fail to improve
over target-only learning, and adaptive transfer can be substantially harder
than oracle transfer \cite{hanneke2022nofree,hanneke2023limits}. Recent robust
methods therefore aim to identify useful sources or avoid negative transfer.
For example, the authors of \cite{fan2025robust} study unreliable source data through an
``ambiguity level'' between source and target regression functions and propose
Transfer Around Boundary, which selectively uses source information near the
decision boundary while retaining risk-improvement guarantees.
The authors of \cite{tripuraneni2020theory} provide statistical guarantees for transfer learning from a representation learning perspective. 
Taking a similar representation learning perspective, the authors of \cite{du2021fewshot} establish few-shot bounds that reduce target dependence from the
ambient dimension \(d\) to an intrinsic representation dimension \(k\) once the
representation is learned. The authors of \cite{tripuraneni2021provable} provide related
upper and lower bounds for meta-learning linear representations. More recent
work relaxes exact sharing assumptions by allowing similar but non-identical
representations or outlier tasks \cite{duan2023adaptive,tian2025similar}.
Casting multi-task bandit problems into representation learning, the authors of \cite{russo2023sample} derive instance-specific lower bounds for sample complexity. The authors of \cite{brunskill2013sample} show that the sample complexity can be significantly reduced by applying transfer learning across multiple reinforcement learning tasks with unknown Markov decision processes; such advantage was also shown in \cite{oguni2014reducing}. 
The authors of \cite{agarwal2023rl} analyze representational transfer in reinforcement learning and showed that source tasks can help learn features that allow target-task algorithms to approach the sample complexity achievable with known ground-truth representations. In contextual bandits, the authors of \cite{cai2024bandits} derive minimax regret rates under covariate shift and quantified how auxiliary source data contribute to target-domain learning. The authors of \cite{mahajan2025learning} establish rigorous bounds for the sample efficiency advantage by employing transfer learning for linear reinforcement learning problems. Transfer has also been studied in structured high-dimensional models, where
source data can estimate a shared component and target data need only estimate a
sparse or low-dimensional correction
\cite{li2022highdim,tian2023glm,liu2024unified,he2024transfusion}. In modern
pretraining settings, recent theory shows that unlabeled or source-task data can
reduce downstream labeled sample complexity when they recover reusable features
\cite{ge2024provable,jonesmccormick2025provable,tahir2024features}. Related
sample-efficiency guarantees have also been developed for conditional diffusion
models \cite{cheng2025diffusion}.
The authors of \cite{rigollet2025on} adopt an entropic optimal transport approach to solve domain adaptation problems and established the corresponding sample complexity. The authors of \cite{vural2025unified} derive the sample complexity for domain adaptation for neural networks training. The authors of \cite{ren2025partial} provide guarantees on error rates and sample requirements in the presence of domain adaptation via importance sampling-based corrections.

\paragraph{Statistical optimal transport (OT).} It has emerged as a powerful framework
to compare, align, and analyze probability distributions.
Building on the foundational work of \cite{Villani09},
which formalized OT theory and its geometric structure,
a growing body of research has focused on statistical aspects such as estimation, convergence rates, and computational efficiency.
Especially, to address the computational challenges of OT,
regularized approaches such as entropic OT \cite{cuturi2013sinkhorn} and sliced OT \cite{bonneel2015sliced} have been developed, yielding scalable algorithms with statistical guarantees.
Applications span a wide range of domains, including generative modeling \cite{de2021diffusion}, domain adaptation \cite{courty2016optimal}, and causal inference \cite{torous2021optimal}, where OT provides both a metric for distributional comparison and a tool for constructing statistically principled estimators.
Recent directions explore robust OT, and connections with kernel methods \cite{MB24}, further cementing OT’s role as a unifying statistical framework for modern data analysis.
Refer to \cite{CNR24} for a textbook account.

\paragraph{Our work.} 
We take an optimal transport viewpoint of transfer learning to analyze its impact on sample complexity.
We interpret transfer learning as properly transporting the feature extraction mechanism from an resolved learning task to a new one. By casting direct learning as nonparametric estimation, we are able to provide the sample efficiencies for both approaches. In particular, we find that when the dimension of the input data is fixed, what influences the sample complexity the most is the smoothness of the optimal target model for direct learning, or the smoothness of input data distribution for transfer learning; a better smoothness condition implies a better sample complexity.

We conduct numerical experiments on image classification in two settings: the standard Office-31 dataset and the diagnosis of retinopathy of prematurity (ROP). The results show that transfer learning consistently outperforms direct learning across all levels of data availability and under all evaluation metrics, aligning with our theoretical analysis on sample complexity. The advantage of transfer learning becomes particularly pronounced as the amount of available data decreases. Details of the experiments are available on \href{https://github.com/anony4git/tl-sample-complexity.git}{GitHub}.

\paragraph{Organization of the paper.} 
We present the transfer learning problem setting and  preliminary analysis in Section~\ref{sc2}, where the general setup of nonparametric estimation can be found in Section \ref{sc21} and the corresponding transfer learning formulation is in Section \ref{sc22}.
The theoretical result is in Section \ref{sc3}
and the numerical experiments  in Section \ref{sc4}.


\paragraph{Notations.}
Throughout the paper, we specify the following notations.
\begin{itemize}[topsep=0pt,itemsep=0pt,partopsep=0pt, parsep=0pt]
\item
For a random variable $X$, $P_X$ denotes the probability distribution of $X$, and 
$\mathbb{E}X$ denotes its expectation.
\item
For $f$ a function on $U$, and $\mu$ a probability distribution on $U$,
$f_{\#} \mu$ denotes the pushforward of $\mu$ by $f$.
\item
For a function $f: \mathbb{R}^d \to \mathbb{R}$, $\nabla f: = (\partial_1 f \ldots. \partial_d f)$ denotes the gradient of $f$.
\item
For $\alpha > 0$, let $[\alpha]$ the integer part of $\alpha$.
We say that a function $f$ is of class $\mathcal{C^{\alpha}}$ if $f$ is $[\alpha]$-times continuous differentiable,
and the $[\alpha]^{th}$ derivatives of $f$ is $(\alpha - [\alpha])$-H\"older continuous.
\end{itemize}

\section{Problem Setting and Preliminary Analysis}
\label{sc2}

In this section, we fix  a problem setting for transfer learning, as in \cite{Cao23}. For  simplicity and without much loss of generality, we focus on a supervised learning framework. 

We denote \((X_T,Y_T)\) for the pair of input and output random variables taking values in \(\XCal_T\times\YCal_T\) for a target task \(T\), where  \(\XCal_T\) is assumed to be compact subset of \(\R^d\) with \(d\in\N^+\) and \(\YCal_T=\R\). The  loss for any admissible target model \(f_T\) is denoted as \(\LL_T(f_T)=\E_{X_T,Y_T}[L_T(f_T(X_T),Y_T)]\).  Here the loss function \(L_T:\YCal_T\times\YCal_T\to[0,+\infty) \) with \(L_T(y,y)=0\) for any \(y\in\YCal_T\). Note that the target loss \(\LL_T\) is subject to the joint distribution of \((X_T,Y_T)\), and we assume this joint distribution to admit an absolutely continuous density.

Let \(f^*_T\) be an optimal target model among the set of admissible target models \(A_T\subset\{f|f:\XCal_T\to\YCal_T\}\) that minimizes the loss \(\LL_T\), then one may assume that the optimal target model is given by  
\begin{equation}
\label{eq:fstar}
f_T^*(x)=\mathbb{E}(Y_T \,|\, X_T = x), \quad x \in \mathcal{X}_T.
\end{equation}
This  holds under a general class of loss functions  \(L_T:\YCal_T\times\YCal_T\to\R\) called  Bregman loss function, including mean squared error and (generalized) Kullback-Leibler divergence  \cite{banerjee2005optimality}.

Under a direct learning approach, \(f_T^*\) is estimated using  \(m\in\N^+\) i.i.d. samples of \((X_T,Y_T)\) available in the dataset \(\mathcal D_T=\{(X_T^i,Y_T^i)\}_{i=1}^m\). 
In contrast, transfer learning is to construct an estimate of \(f_T^*\) by exploiting some already resolved problem,  called the source task. Let \((X_S,Y_S)\in\XCal_S\times\YCal_S\) be the pair of input and output random variables for \(S\) with \(\XCal_S=\XCal_T\) and \(\YCal_S=\YCal_T\), let \(f_S^*:\XCal_S\to\YCal_S\) be the optimal source model, called the pretrained model, that minimizes some loss \(\LL_S\) which is subject to the joint distribution of \((X_S,Y_S)\). We also assume that the joint distribution of \((X_S,Y_S)\) admits an absolutely continuous density function. 

When the source task is sufficiently similar to the target task, the pretrained model \(f_S^*\) is believed to contain a feature extractor unit that is suitable for both tasks, in the following sense:
\begin{equation}\label{eq:tl-idea}
f^*_T\approx\mathcal T^{\YCal*}_{S\to T}\circ f_S^*\circ\mathcal T^{\XCal*}_{T\to S};\end{equation} 
here, \(\mathcal T^{\XCal*}_{T\to S}:\XCal_T\to\XCal_S\) aims to transport \(Law(X_T)\) to \(Law(X_S)\), and \(\mathcal T^{\YCal*}_{S\to T}\) aims to transport \(Law(f_S^*(X_S))\) to \(Law(f_T^*(X_T))\).
It is thus natural to adopt an optimal transport viewpoint to analysis sample efficiency of such a transfer learning approach. (See \cite{Villani03, Villani09} for more details). We refer to \eqref{eq:tl-fw} as an illustration of this transfer learning framework.
\begin{equation}
    \label{eq:tl-fw}
    \begin{matrix}
    \XCal_S\ni X_S & \xRightarrow{\text{\hspace{2pt} Pretrained model } f_S^* \text{\hspace{2pt}}} & f_S^*(X_S)\in\YCal_S\\
       \mathcal T^{\XCal*}_{T\to S}\Big\Uparrow & & \Big\Downarrow \mathcal T_{S\to T}^{\YCal*} \\
        \XCal_T\ni X_T & \xdashrightarrow[\text{\hspace{2pt}}f_T^*\in\argmin_{f\in A_T} \LL_T(f_T)\text{\hspace{2pt}}]{\text{Direct learning }} & f_T^*(X_T)\in\YCal_T
    \end{matrix}
\end{equation}


\subsection{Nonparametric Estimation in Direct Learning}
\label{sc21}
In direct learning, it is to characterize the optimal relation $f_T^*: \mathcal{X}_T \to \mathcal{Y}_T$ prescribed in \eqref{eq:fstar} between $X_T$ and $Y_T$ using the sample data \(\mathcal D_T=\{(X_T^i,Y_T^i)\}_{i=1}^m\).
Such an estimation can be done by parametric or nonparametric methods.
A classical result of Stone \cite{Stone80, Stone82} shows how well
the regression function $f_T^*$ can be estimated with respect to the number of sample data \(m\).
\begin{theorem}\cite{Stone80, Stone82}
\label{thm:stone}
Assume that $f_T^*(\cdot)$ in \eqref{eq:fstar} is $p$-times differentiable. 
Then there exists an estimate $\hat{f}_m(\cdot)$ (based on $m$ samples) such that
\begin{equation}
\label{eq:stonest}
\mathbb{E}|\hat{f}_m - f_T^*|_{L^2} \le C m^{-\frac{p}{2p+d}},
\end{equation}
for some $C > 0$ (independent of $m$).
Moreover, the rate $m^{-\frac{p}{2p+d}}$ on the right side of \eqref{eq:stonest} is optimal
in some asymptotic sense.
\end{theorem}

Refer to \cite[p.1042]{Stone82} for the precise definition of optimality stated in Theorem \ref{thm:stone}.
Note that the sample complexity \eqref{eq:stonest} of nonparametric estimation
depends on:
\begin{itemize}[topsep=0pt,itemsep=0pt,partopsep=0pt, parsep=0pt]
\item
the dimension $d$ of the target space $\mathcal{X}_T$,
\item
the smoothness $p$ of the regression function $f_T^*(x):= \mathbb{E}(Y_T \,|\, X_T =x)$.
\end{itemize}

\subsection{Transfer Learning via Optimal Transport}
\label{sc22}

When data is scarce, adopting a direct learning approach becomes challenging, especially when the training sample size 
$m$ is too small relative to the model’s complexity. In such cases, transfer learning offers a powerful alternative, by leveraging a pretrained model that has already been trained on a much larger dataset.

To formalize the transfer learning idea in \eqref{eq:tl-idea},  consider the following optimization problem,
\begin{equation}
\label{eq:OTP}
\min_{\mathcal{T}^{\mathcal{X}}_{T \to S}, \mathcal{T}^{\mathcal{Y}}_{S \to T}}\mathbb{E}\bigg(Y_T - \mathcal{T}^{\mathcal{Y}}_{S \to T} \circ f_S^* \circ \mathcal{T}^{\mathcal{X}}_{T \to S} (X_T) \bigg)^2,
\end{equation}
where
$\mathcal{T}^{\mathcal{X}}_{T \to S}$
(resp. $\mathcal{T}^{\mathcal{Y}}_{S \to T}$)
is a transfer map from $\mathcal{X}_T$ to $\mathcal{X}_S$
(resp. from $\mathcal{Y}_S$ to  $\mathcal{Y}_T$).
The transfer learning \eqref{eq:OTP} can be viewed as 
a constrained optimization problem, 
with the transfer map 
$ \mathcal{X}_S \to \mathcal{Y}_S$
set to be \(f_S^*\). 
This transfer learning is inherently different from optimal transport between different dimensions \cite{MP20, Pass12}.
Second, it is possible to use a general loss function $L_T(\cdot,\cdot)$, which yields the objective:
\begin{equation*}
\min_{\mathcal{T}^{\mathcal{X}}_{T \to S}, \mathcal{T}^{\mathcal{Y}}_{S \to T}}\mathbb{E}L_T\left(\mathcal{T}^{\mathcal{Y}}_{S \to T} \circ f_S^* \circ \mathcal{T}^{\mathcal{X}}_{T \to S} (X_T),\, Y_T \right).
\end{equation*}
The problem \eqref{eq:OTP} corresponds to the quadratic loss $L_T(y,y') = (y-y')^2$ for any \(y,y'\in\YCal_T\),
and other popular choices of the loss includes $L_T(y,y') = |y-y'|^q$ for $q > 1$. 
For ease of presentation, 
we focus on the quadratic loss, and the learning problem \eqref{eq:OTP} in the sequel.
Finally, for implementation we need to replace the expectation in \eqref{eq:OTP} with the empirical average:
\begin{equation}
\label{eq:OTPem}
\min_{\mathcal{T}^{\mathcal{X}}_{T \to S}, \mathcal{T}^{\mathcal{Y}}_{S \to T}} \frac{1}{m} \sum_{i=1}^m \bigg(Y^i_T - \mathcal{T}^{\mathcal{Y}}_{S \to T} \circ f_S^* \circ \mathcal{T}^{\mathcal{X}}_{T \to S} (X^i_T)\bigg)^2.
\end{equation}

The following lemma gives an easy sufficient condition for the problem \eqref{eq:OTP} to have a solution.
\begin{lemma}
\label{lem:ext}
Let $f_T^*(\cdot)$ be defined by \eqref{eq:fstar}, and assume that there exist transfer maps
$\mathcal{T}^{\mathcal{X}*}_{T \to S}$, $\mathcal{T}^{\mathcal{Y}*}_{S \to T}$ such that
\begin{equation}
\label{eq:assumpext}
f_T^* = \mathcal{T}^{\mathcal{Y}*}_{S \to T} \circ f_S^* \circ \mathcal{T}^{\mathcal{X}*}_{T \to S}.
\end{equation}
Then $(\mathcal{T}^{\mathcal{X}*}_{T \to S}, \mathcal{T}^{\mathcal{Y}*}_{S \to T})$ is a solution to the problem \eqref{eq:OTP}.
\end{lemma}
\begin{proof}
The result follows from the fact that $f_T^*(X_T):= \mathbb{E}(Y_T \,|\, X_T)$ is the orthogonal projection 
of $Y_T$ on the $\sigma$-field generated by $X_T$.
\end{proof}

Lemma \ref{lem:ext} provides the existence of a solution to the learning problem \eqref{eq:OTP}.
In fact, there may be many (possibly infinitely many) solutions to the problem \eqref{eq:OTP},
all of which are equally ``good" because they lead to the same optimal prediction $f_T^*(\cdot)$.
However, the non-uniqueness of solutions makes the analysis subtle:
we need to choose a ``canonical" solution as the ground truth for analysis. 
Here our idea is to use optimal transport theory,
by choosing $(\mathcal{T}^{\mathcal{X}*}_{T \to S}, \mathcal{T}^{\mathcal{Y}*}_{S \to T})$
that are the (unique) solutions to some transportation problems.
To this end, we make the following assumption.
\begin{assump}
\label{assump:OTP}
There exist transfer maps $\mathcal{T}^{\mathcal{X}*}_{T \to S}$, $\mathcal{T}^{\mathcal{Y}*}_{S \to T}$,
and a random variable $X_S$ on $\mathcal{X}_S$ such that:
\begin{enumerate}[topsep=0pt,itemsep=0pt,partopsep=0pt, parsep=0pt]
\item
The relation \eqref{eq:assumpext} holds.
\item
The map $\mathcal{T}^{\mathcal{X}*}_{T \to S}$ solves the Monge problem:
\begin{equation}
\label{eq:MG1}
\min_{P_{X_S} = \mathcal{T}_{\#} P_{X_T}} \mathbb{E}(\mathcal{T}X_T - X_S)^2.
\end{equation}
\item
Denote $\bar Y_S\coloneqq f_S^*(X_S)$ and $\bar Y_T\coloneqq f_T^*(X_T^*)$, so that $P_{\bar Y_S}\coloneqq {f_S^*}_{\#} P_{X_S}$ and $P_{\bar Y_T}\coloneqq {f_T^*}_{\#}P_{X_T}$.
The map $\mathcal{T}^{\mathcal{Y}*}_{S \to T}$ solves the Monge problem:
 \begin{equation}
 \label{eq:MG2}
\min_{P_{\bar Y_T} = \mathcal{T}_{\#} P_{\bar Y_S}} \mathbb{E}(\mathcal{T}\bar Y_S - {\bar Y_T})^2.
\end{equation}
\end{enumerate}
\end{assump}
With the above assumptions, the following proposition resolves the non-uniqueness issue.
\begin{proposition}
\label{prop:unique}
The transfer maps $(\mathcal{T}^{\mathcal{X}*}_{T \to S}, \mathcal{T}^{\mathcal{Y}*}_{S \to T})$ in Assumption \ref{assump:OTP} solve the problem \eqref{eq:OTP},
and are unique almost everywhere.
\end{proposition}
\begin{proof}
The fact that $(\mathcal{T}^{\mathcal{X}*}_{T \to S}, \mathcal{T}^{\mathcal{Y}*}_{S \to T})$ solve the problem \eqref{eq:OTP}
follow from Lemma \ref{lem:ext}.
By Brenier's theorem  \cite[Chapter 3]{Villani03},
$\mathcal{T}^{\mathcal{X}*}_{T \to S}$ (resp. $\mathcal{T}^{\mathcal{Y}*}_{S \to T}$) is the unique (in the sense of almost everywhere)
optimal transport map to the Monge problem in (2) (resp. in (3)).
\end{proof}

Proposition \ref{prop:unique} shows that Assumption \ref{assump:OTP} provides a canonical optimal solution
to the transfer learning problem \eqref{eq:OTP}, with additional optimal transport constraints \(P_{X_S}={\mathcal T^{\XCal}_{T\to S}}_{\#}P_{X_T}\) and \(P_{\bar Y_T}={\mathcal T^{\YCal}_{S\to T}}_{\#}P_{\bar Y_S}\).
Note that the quadratic costs in Assumption \ref{assump:OTP} (2)-(3) can be replaced with any strictly convex costs
(see \cite{GM96}).
Moreover, each of $(\mathcal{T}^{\mathcal{X}*}_{T \to S}, \mathcal{T}^{\mathcal{Y}*}_{S \to T})$,
as optimal transport maps,
can be expressed as the gradient of a convex function
that solves a Monge-Amp\`ere equation \cite{Brenier87, Brenier91}.
See also \cite{Ca922, Ca92, Ca96} for the regularity theory of optimal transport maps.


\section{Sample Complexity of Transfer Learning}
\label{sc3}

In this section, we consider the sample complexity of the transfer learning problem.
To establish the sample complexity result, 
we need some further assumptions.
\begin{assump}
\label{assump:OTP2}
Let $(\mathcal{T}^{\mathcal{X}*}_{T \to S}, \mathcal{T}^{\mathcal{Y}*}_{S \to T})$ 
be specified by Assumption \ref{assump:OTP}.
\begin{enumerate}[topsep=0pt,itemsep=0pt,partopsep=0pt, parsep=0pt]
\item
There exist $C > c > 0$ such that 
$c \le |\nabla f_S^*| \le C$.
\item
$\mathcal{T}^{\mathcal{Y}*}_{S \to T}$ is Lipschitz.
\end{enumerate}
\end{assump}

Generally, the optimal transport maps are not Lipschitz. 
A sufficient condition for $\mathcal{T}^{\mathcal{Y}*}_{S \to T}$ to be Lipschitz is that
both $P_{Y_T}$ and $P_{Y_S}$ are log-concave -- this result is known as Caffarelli's contraction theorem \cite{Car00}, 
(see also \cite{CFS24, CFJ17} for extensions).

Our main result is stated as follows. 
The proof relies on the recently developed statistical theory of optimal transport, 
see e.g., \cite{CNR24, MB24}.

\begin{theorem}
\label{thm:main}
Let Assumption \ref{assump:OTP}--\ref{assump:OTP2} hold.
Suppose that the densities of $P_{X_T}$ and $P_{X_S}$ are of class $\mathcal{C}^\alpha$ for some $\alpha > 0$.
Then there exist transfer maps $(\hat{\mathcal{T}}^{\mathcal{Y},m}_{S \to T},  \hat{\mathcal{T}}^{\mathcal{X},m}_{T \to S})$
(based on $m$ samples) such that 
\begin{equation}
\label{eq:keybd}
\mathbb{E}|\hat{\mathcal{T}}^{\mathcal{Y},m}_{S \to T} \circ  f_S^* \circ \hat{\mathcal{T}}^{\mathcal{X},m}_{T \to S} - f_T^*|_{L^2} \le C(\kappa_m + m^{-\frac{1}{2}}),
\end{equation}
for some $C > 0$ (independent of $m$), where
\begin{equation}
\label{eq:kappam}
\kappa_m: =  \left\{ \begin{array}{rcl}
m^{-\frac{1}{2}} & \mbox{for}& d = 1, \\ 
\sqrt{\log m / m} & \mbox{for} & d = 2, \\
m^{-\frac{\alpha + 1}{2 \alpha + d}} & \mbox{for} & d \ge 3.
\end{array}\right.
\end{equation}
\end{theorem}
\begin{proof}
By Assumption \ref{assump:OTP}, we have:
$$f_T^* = \mathcal{T}^{\mathcal{Y}*}_{S \to T} \circ f_S^* \circ \mathcal{T}^{\mathcal{X}*}_{T \to S},$$
where $\mathcal{T}^{\mathcal{X}*}_{T \to S}$ (resp. $ \mathcal{T}^{\mathcal{Y}*}_{S \to T}$)
is the optimal transport map of the problem \eqref{eq:MG1} (resp. \eqref{eq:MG2}).
For the moment, 
let $\hat{\mathcal{T}}^{\mathcal{X},m}_{T \to S}$ (resp. $\hat{\mathcal{T}}^{\mathcal{Y},m}_{S \to T}$)
be any empirical optimal transport map to the problem \eqref{eq:MG1} (resp. \eqref{eq:MG2}),
based on  $m$ i.i.d. samples from the distribution $P_{(X_T, X_S)}$ (resp. $P_{(Y_S, Y_T)}$).
Denote by
\begin{equation}
\hat{f}_m:=\hat{\mathcal{T}}^{\mathcal{Y},m}_{S \to T} \circ  f_S^* \circ \hat{\mathcal{T}}^{\mathcal{X},m}_{T \to S},
\end{equation}
the prediction via transfer learning.
By the triangle inequality, we have:
\begin{equation}
\label{eq:aplusb}
\mathbb{E}|\hat{f}_m - f_T^*|_{L^2} \le (a) + (b),
\end{equation}
where 
\begin{equation*}
\begin{aligned}
& (a):= \mathbb{E}|\mathcal{T}^{\mathcal{Y}*}_{S \to T} \circ  f_S^* \circ \hat{\mathcal{T}}^{\mathcal{X},m}_{T \to S}- \mathcal{T}^{\mathcal{Y}*}_{S \to T} \circ f_S^* \circ \mathcal{T}^{\mathcal{X}*}_{T \to S}|_{L^2}, \\
& (b):= \mathbb{E} |\hat{\mathcal{T}}^{\mathcal{Y},m}_{S \to T} \circ  f_S^* \circ \hat{\mathcal{T}}^{\mathcal{X},m}_{T \to S}- \mathcal{T}^{\mathcal{Y}*}_{S \to T} \circ  f_S^* \circ \hat{\mathcal{T}}^{\mathcal{X},m}_{T \to S}|_{L^2}.
\end{aligned}
\end{equation*}
We proceed to estimating the terms (a) and (b).

\smallskip
\fbox{Term (a)} 
By the Lipschitz assumption on $\mathcal{T}^{\mathcal{Y}*}_{S \to T}$ and $f_S^*$,
we get:
\begin{equation*}
|\mathcal{T}^{\mathcal{Y}*}_{S \to T} \circ  f_S^* \circ \hat{\mathcal{T}}^{\mathcal{X},m}_{T \to S} - \mathcal{T}^{\mathcal{Y}*}_{S \to T} \circ f_S^* \circ \mathcal{T}^{\mathcal{X}*}_{T \to S}|_{L^2}
\le C_1 |\hat{\mathcal{T}}^{\mathcal{X},m}_{T \to S} - \mathcal{T}^{\mathcal{X}*}_{T \to S}|_{L^2},
\end{equation*}
for some $C_1 > 0$.
By \cite[Theorem 18]{MB24}, there exists a (kernel) transport map $\hat{\mathcal{T}}^{\mathcal{X},m}_{T \to S}$
such that
$\mathbb{E}|\hat{\mathcal{T}}^{\mathcal{X},m}_{T \to S} - \mathcal{T}^{\mathcal{X}*}_{T \to S}|_{L^2} \le C_2 \kappa_m$
for some $C_2 > 0$,
where $\kappa_m$ is defined by \eqref{eq:kappam}.
This yields the estimate:
\begin{equation}
\label{eq:terma}
(a) \le C_3 \kappa_m, \quad \mbox{for some } C_3 > 0 \mbox{ (independent of $m$)}.
\end{equation}

\smallskip
\fbox{Term (b)} 
Denote by $g_m:= f_S^* \circ \hat{\mathcal{T}}^{\mathcal{X},m}_{T \to S}$.
We can check that 
$\inf_x \left| \partial_1 g_m \right| > C_4$ for some $C_4 > 0$ (independent of $m$).
Next we complete $g_m(x)$ with $(g_m(x), x_2, \ldots, x_n)$ to form a basis,
and obtain:
\begin{equation*}
\begin{aligned}
&|\hat{\mathcal{T}}^{\mathcal{Y},m}_{S \to T} \circ  f_S^* \circ \hat{\mathcal{T}}^{\mathcal{X},m}_{T \to S}- \mathcal{T}^{\mathcal{Y}*}_{S \to T} \circ  f_S^* \circ \hat{\mathcal{T}}^{\mathcal{X},m}_{T \to S}|_{L^2}\\
& = |\hat{\mathcal{T}}^{\mathcal{Y},m}_{S \to T} \circ  g_m - \mathcal{T}^{\mathcal{Y}*}_{S \to T} \circ  g_m|_{L^2} \\
& \le C_5 |\hat{\mathcal{T}}^{\mathcal{Y},m}_{S \to T} -\mathcal{T}^{\mathcal{Y}*}_{S \to T} |_{L^2},
\end{aligned}
\end{equation*}
where the inequality follows from a change of variables $x \leftrightarrow (g_m(x), x_2, \ldots, x_n)$,
with the Jacobian matrix of form:
\begin{equation*}
\begin{pmatrix}
(\partial_1 g_m)^{-1} & 0 &  \cdots & 0\\
* & 1 &  \cdots & 0\\
\vdots & \vdots & \ddots & \vdots  \\
* & 0 &  \cdots & 1
\end{pmatrix},
\end{equation*}
and the fact that $\mathcal{X}_T$ is a compact subset of $\mathbb{R}^d$.
Again by \cite[Theorem 18]{MB24} (specializing to $d = 1$),
there exists a (kernel) transport map $\hat{\mathcal{T}}^{\mathcal{Y},m}_{S \to T}$
such that
$\mathbb{E}|\hat{\mathcal{T}}^{\mathcal{Y},m}_{S \to T} -\mathcal{T}^{\mathcal{Y}*}_{S \to T} |_{L^2} \le C_6/\sqrt{m}$ 
for some $C_6> 0$.
As a result,
\begin{equation}
\label{eq:termb}
(b) \le \frac{C_7}{\sqrt{m}}, \quad \mbox{for some } C_7 > 0 \mbox{ (independent of $m$)}.
\end{equation}

Combining \eqref{eq:aplusb}, \eqref{eq:terma} and \eqref{eq:termb} yields the desired  result.
\end{proof}

{\bf Remark on Theorem \ref{thm:main}.}
Of particular interest is the high-dimensional case ($d \to \infty$), 
where the bound \eqref{eq:keybd} leads to:
\begin{equation}
\mathbb{E}|\hat{\mathcal{T}}^{\mathcal{Y},m}_{S \to T} \circ  f_S^* \circ \hat{\mathcal{T}}^{\mathcal{X},m}_{T \to S} - f_T^*|_{L^2} \le Cm^{-\frac{\alpha + 1}{2 \alpha + d}}.
\end{equation}
Note that the convergence rate depends on:
\begin{itemize}[topsep=0pt,itemsep=0pt,partopsep=0pt, parsep=0pt]
\item
the dimension $d$ of the target space $\mathcal{X}_T$,
\item
the smoothness $\alpha$ of the target distribution (not on the smoothness $p$ of the regressor $\mathbb{E}(Y_T \,|\, X_T = x)$).
\end{itemize}
Comparing with Theorem \ref{thm:stone}, \(L^2\) errors for the learned models are upper bounded by
\begin{equation}
 \left\{ \begin{array}{lll}
O\left(m^{-\frac{p}{d}}\right) & \mbox{for direct approach}, \\ 
O\left(m^{-\frac{\alpha + 1}{d}} \right) & \mbox{for transfer learning approach}.
\end{array}\right.
\end{equation}
So in the direct approach, the exponent is proportional to $\frac{p}{d}$,
where $p$ is the smoothness of the regression kernel $\mathbb{E}(Y_T \,|\, X_T = x)$;
while in the transfer learning case,
only the smoothness $\alpha$ of the source/target distribution matters.

A typical scenario where the transfer learning approach has a better sample complexity is 
that the regressor $\mathbb{E}(Y_T \,|\, X_T = x)$ is lesser smooth,
e.g., $p = 1$ that yields the rate $m^{-\frac{1}{d}}$;
while the source/target distributions are much smoother $(\alpha \to +\infty)$.
In particular, if
\begin{equation}
\label{eq:TLbetter}
\alpha + 1 > p,
\end{equation}
then the transfer learning approach is advantageous over the direct approach 
in terms of the sample complexity. In practice, the above situation can occur when the goal of the target task is to learn a highly complex network containing non-smooth activation functions, while the input data distributions are approximately of exponential family.

\section{Numerical Experiments}
\label{sc4}
In this section, we use image classification to study   sample complexity between transfer learning and direct learning.
We analyze numerically two cases. The first is the classification problem using the benchmark dataset, Office-31.
The second is the classification problem in diagosis for one of the eye diseases for pre-mature infants called retinopathy of prematurity (ROP). In both cases, we vary the amount of the corresponding target data used for the training data, and compare the model performance between direct learning and transfer learning.
Our experiments results are shown to be consistent with the theoretical results presented in Sections~\ref{sc3}. In particular, when data is scarce, the advantage of transfer learning over direct learning becomes most pronounced.

\paragraph{Evaluation metrics and neural network setup.} For both direct and transfer learning, we adopt the following four performance metrics for the ROP diagnosis results: 
\begin{enumerate}[topsep=0pt,itemsep=0pt,partopsep=0pt, parsep=0pt]
    \item AUROC   evaluates  classification models' ability to discriminate between positive and negative classes across all thresholds,
    \item Accuracy is the ratio of correct predictions out of all predictions, i.e., \(\frac{{\rm True Positives}+{\rm True Negatives}}{{\rm All \,\, predictions}}\),
    \item Precision is the ratio of true positives over all reported positives, i.e., \(\frac{{\rm True Positives}}{{\rm True Positives} + {\rm False Positives}}\), and
    \item Sensitivity is the number of true positives detected over the total number of positive cases in the test dataset, i.e., \(\frac{{\rm True Positives}}{{\rm True Positives} + {\rm False Negatives}}\).
\end{enumerate}
For each metric, the numerical value ranges in \([0,1]\), with higher values indicating  better learning outcome. An AUROC score of 1.0 indicates a perfect model, and 0.5 indicating random guessing.
We also report the average relative performance differences between transfer learning  and direct learning, where
{\footnotesize$$\Delta \text{metric}=\frac{\text{metric}_{\text{TL}}-\text{metric}_{\text{direct learning}}}{\text{metric}_{\text{direct learning}}}\times 100\%.$$}

The neural network for classification consists of three units: the resizing unit, the feature extractor unit implemented via \texttt{ResNet50}, and a classifier layer.

\subsection{Image Classification on the Office-31 Dataset}\label{subsec:office}

To study the sample complexity for transfer learning versus direct learning, we start with a standard  image classification problem with  a benchmark dataset Office-31 \cite{saenko2010adapting}.

In this dataset, there are three domains: Amazon (A), Webcam (W), and DSLR (D), containing 4114 images of 31 categories of objects in an office environment.  In this experiment, images from domains D and W are seen as the target data, with a total number of \({\rm total}=1296\). When applying transfer learning, illustrated in Figure \ref{fig:office-tl}, the feature extractor, with input dimension \(3\times 244\times 244\) and output dimension \(2048\), is pretrained using either ImageNet (in which case the weights are publicly available in \cite{chollet2015keras}) or domain A as the source data. The resizing unit and the classifier layer are retrained as the input and output transport mappings, respectively. 
\begin{figure*}[!ht]
    \centering
    \includegraphics[width=0.75\textwidth]{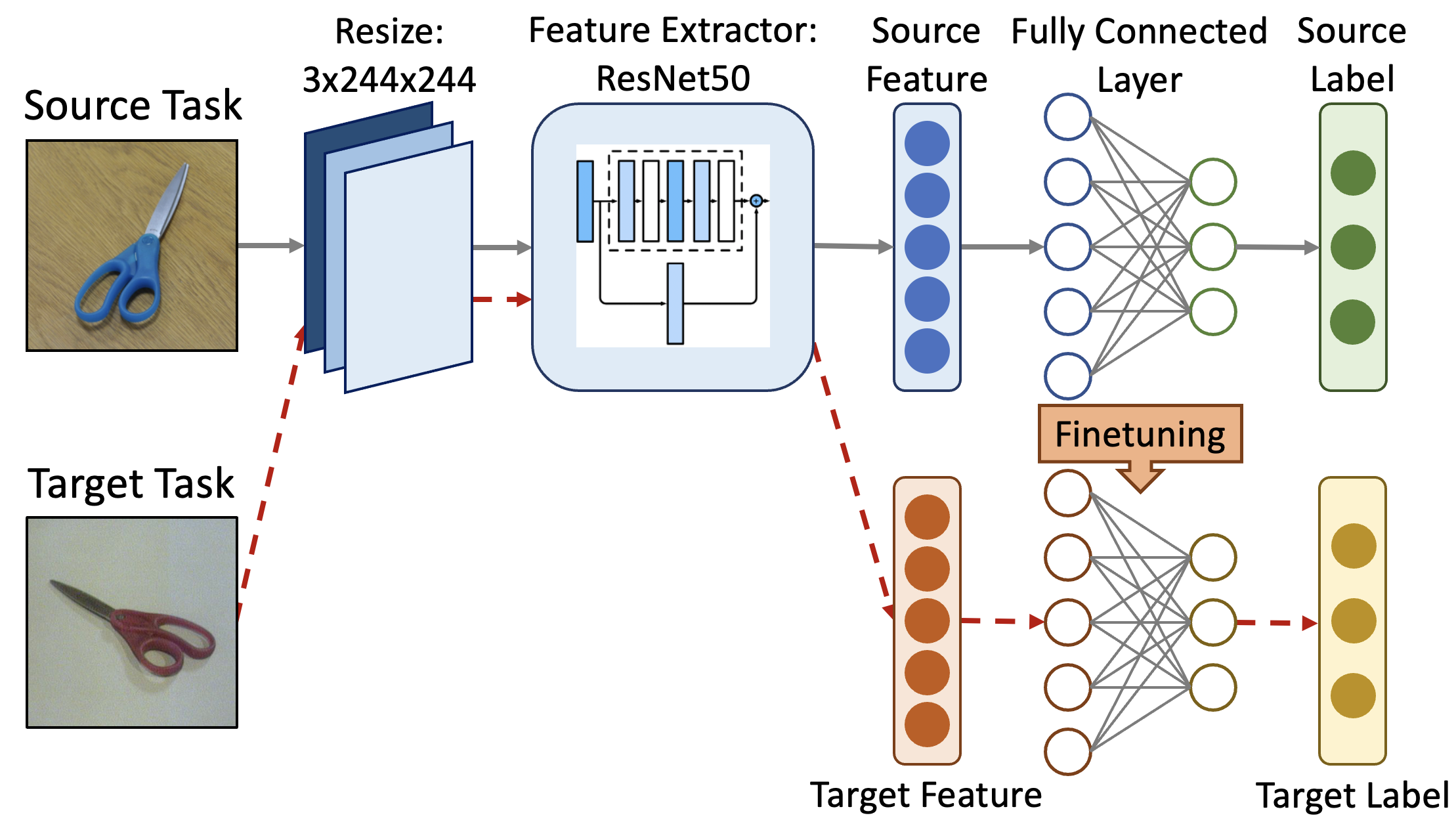}
    \caption{Illustration of the Office-31 transfer learning task.}
    \label{fig:office-tl}
\end{figure*}

By the direct learning approach, the neural network is trained with a random initialization. We vary the percentage of the target data (i.e. D+W Data \(\%\))  being used for both the training in direct learning and the fine-tuning in transfer learning, so that \(m={\rm total}\times \text{D+W Data }\%\).

\paragraph{Results.
}
The average performance evaluation for both direct and transfer learning is summarized in Table \ref{tab:office}. It is clear  that transfer learning, either using ImageNet or domain A as source data, outperforms direct learning across all metrics, and at any given level of target data availability.
At the data-scarce regime, i.e., \(\text{D+W Data \%}\leq30\%\), transfer learning significantly dominates direct learning. In the extreme case with \(10\%\) of the target data in the training set, direct learning achieves only \(0.3\) in Precision, while the transfer learning  achieves over \(0.6\). 
\begin{table*}[!ht]
\centering
\caption{Office-31 Classification Performance Evaluation}
\begin{adjustbox}{max width=\linewidth}
\begin{tabular}{c|ccc|ccc|ccc|ccc}
\toprule
\multirow{2}{*}{D+W Data \% (m)} 
& \multicolumn{3}{c|}{AUROC} 
& \multicolumn{3}{c|}{Accuracy} 
& \multicolumn{3}{c|}{Precision} 
& \multicolumn{3}{c}{Sensitivity} \\
\cmidrule{2-13}
& Direct & Transfer-A & Transfer-ImageNet 
& Direct & Transfer-A & Transfer-ImageNet 
& Direct & Transfer-A & Transfer-ImageNet 
& Direct & Transfer-A & Transfer-ImageNet \\
\midrule
100\% (1296) & 1.00 & 1.00 & 1.00 & 0.97 & 0.99 & 1.00 & 0.98 & 0.99 & 1.00 & 0.97 & 0.99 & 1.00 \\
50\% (648) & 1.00 & 1.00 & 1.00 & 0.90 & 0.94 & 1.00 & 0.91 & 0.95 & 1.00 & 0.90 & 0.95 & 1.00 \\
20\% (259) & 0.95 & 0.98 & 0.99 & 0.57 & 0.71 & 0.85 & 0.56 & 0.68 & 0.88 & 0.56 & 0.70 & 0.86 \\
10\% (130) & 0.87 & 0.94 & 0.96 & 0.35 & 0.60 & 0.72 & 0.30 & 0.60 & 0.72 & 0.34 & 0.59 & 0.73 \\
\bottomrule
\end{tabular}
\end{adjustbox}
\label{tab:office}
\end{table*}

\begin{table*}[!ht]
    \centering
    \caption{ROP Detection Performance Evaluation}
   \begin{tabular}{c|cc|cc|cc|cc}
    \toprule
         \multirow{2}{*}{ROP Data \% (m)} & \multicolumn{2}{c|}{AUROC} & \multicolumn{2}{c|}{Accuracy} & \multicolumn{2}{c|}{Precision} & \multicolumn{2}{c}{Sensitivity} \\
         \cmidrule{2-9}
          & Direct & Transfer & Direct & Transfer & Direct & Transfer & Direct & Transfer \\
          \midrule
          100\% (5838)  & 0.88 & 0.97 & 0.81 & 0.96 & 0.57 & 0.90 & 0.78 & 0.94 \\
          50\% (2919) & 0.85 & 0.97 & 0.82 & 0.95 & 0.60 & 0.88 & 0.71 & 0.90 \\

10\% (584) & 0.81 & 0.95 & 0.80 & 0.91 & 0.56 & 0.80 & 0.67 & 0.84 \\

1\% (58) & 0.72 & 0.84 & 0.73 & 0.83 & 0.44 & 0.63 & 0.49 & 0.67 \\
\bottomrule
    \end{tabular}
    \label{tab:rop}
\end{table*}

To further demonstrated the benefit of transfer learning,  Table \ref{tab:img-improve} reports the average relative performance differences between transfer learning (ImageNet version) and direct learning. Here the benefit of  transfer learning become increasingly more  prominent  as the target data becomes increasingly scarce. 
Our numerical analysis also shows that the choice of different source task and data may affect the quality of transfer learning. 

The numerical results align with the theoretical findings in Section~\ref{sc3}. From Figure \ref{fig:office-tl}, we see that the target classification model employs a \texttt{ResNet-50}-based feature extractor, containing multiple layers with \texttt{ReLU} as the activation functions; such an architecture reduces the smoothness \(p\) of the optimal target model, i.e., \(p\leq1\). The distribution of images can be approximated as a Gaussian mixture model \cite{permuter2006study} which is smooth, i.e., \(\alpha=+\infty\). These settings make \eqref{eq:TLbetter} hold.

\begin{table*}[!ht]
\centering
\caption{Relative Improvement of Transfer Learning over Direct Learning for Office-31 Classification}
\begin{adjustbox}{max width=\columnwidth}
\begin{tabular}{c|cccc}
\toprule
{D+W Data \(\%\)} & \(\Delta\text{AUROC}\) & \(\Delta\text{Accuracy}\) & \(\Delta\text{Precision}\) & \(\Delta\text{Sensitivity}\) \\
\midrule
100\% & 0.02\% & 3.25\% & 2.07\% & 2.82\% \\
50\% & 0.26\% & 11.40\% & 9.99\% & 10.98\% \\
20\% & 4.56\% & 47.95\% & 55.62\% & 54.99\% \\
10\% & 10.18\% & 109.09\% & 142.69\% & 111.77\% \\
\bottomrule
\end{tabular}
\end{adjustbox}
\label{tab:img-improve}
\end{table*}

\subsection{Detecting ROP from Detections of DR}

We now study the problem of detecting an eye disease called retinopathy of prematurity (ROP).  
To analyze the  sample efficiency of transfer learning compared with direct learning, we compare the diagnosis performances of \(\hat{\mathcal T}^{\mathcal Y,m}_{S\to T}\circ f^*_S\circ \hat{\mathcal T}^{\mathcal X,m}_{T\to S}\) and \(\hat f_m\) as we vary the sample sizes \(m\) of ROP training data. 

\paragraph{Classification problem for ROP.} 
ROP is a common retinal vascular disease primarily affecting prematurely-born infants or infants with low birth weight \cite{Fiersone20183061}, which is one of the leading causes of infant blindness globally \cite{Blencowe2013,gilbert_rahi_eckstein_osullivan_foster_1997,10.1001/archopht.121.12.1684}. To build an efficient, e.g., deep-learning based, ROP detector,  the major obstacles are the lack of sufficient data {\it and} the high accuracy requirement due to the specific case of ROP involving infants.  

Technically, the ROP detection problem is  a binary classification problem, with \(f^*_T\) being the optimal classifier. By direct learning,  the classifier is trained entirely on the ROP dataset which contains \(9727\) anonymized images in total, with \(2310\) positive samples and \(7417\) negative samples. The corresponding classifier after direct learning is \(\hat f_m\).

\paragraph{Transfer learning from DR to ROP.} By the transfer learning approach, we leverage on the  successes of  the deep learning approach for diagnosis of diabetic retinopathy (DR),  another retinal vascular disease for people with diabetes.  The pre-trained models for DP is our choice of source task for the transfer learning study. 

The DR dataset is much larger and contains \(36126\) anonymized images with \(26548\) positive samples and \(9578\) negative samples. 
Among all the 9727 ROP data, \(1945\) images (\(462\) positive samples and \(1483\) negative samples) are used as the validation data and \(1944\) images (\(461\) positive samples and \(1483\) negative samples) are reserved to be the test data. The maximum number of training samples available is \({\rm total}= 5838\) (\(1387\) positive samples and \(4451\) negative samples).

With transfer learning, as shown in  Figure \ref{fig:tl}, the feature extractor is first pretrained on DR dataset, denoted by \(f^*_S\). Then, using the ROP dataset, this feature extractor is fine-tuned with the resizing unit as the input transport mapping \(\hat{\mathcal T}^{\mathcal X, m}_{T\to S}\) and the classifier layer as the output transport mapping \(\hat{\mathcal T}^{\mathcal Y,m}_{S\to T}\).

In the experiment, we vary the percentage of ROP data  being used for both the training in direct learning and the fine-tuning in transfer learning, so that \(m={\rm total}\times \text{ROP Data }\%\).
\begin{figure*}[!ht]
    \centering
    \includegraphics[width=0.85\linewidth]{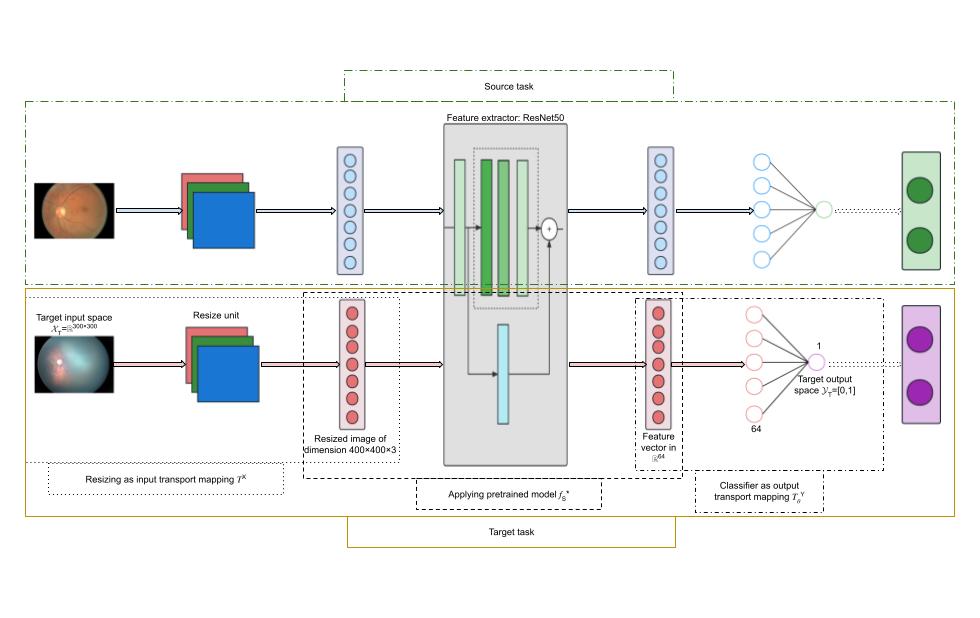}
    \caption{Transfer learning}
    \label{fig:tl}
\end{figure*}

\paragraph{Results.
}

Similar to  the Office-31 example in Section \ref{subsec:office}, 
the numerical findings for ROP  problem detailed  in Tables~\ref{tab:rop}--\ref{tab:meandiff} are (again) consistent with the theoretical results in Sections \ref{sc3}.

\begin{table*}[!ht]
 \centering
 \caption{Relative Improvement of Transfer Learning over Direct Learning for ROP Detection}
 \begin{adjustbox}{max width=\columnwidth}
 \begin{tabular}{c|c|c|c|c}
  \toprule
   ROP Data \(\%\) & \(\Delta\text{AUROC}\) & \(\Delta\text{Accuracy}\) & \(\Delta\text{Precision}\) & \(\Delta\text{Sensitivity}\) \\
   \midrule
   100\% & 10.05\% & 18.75\% & 58.01\% & 19.71\% \\
   
   50\% & 14.09\% & 15.37\% & 45.83\% & 27.19\% \\
   
   10\% & 16.58\% & 14.65\% & 43.68\% & 26.03\% \\

   1\% & 16.08\% & 13.04\%      & 46.67\% & 49.46\% \\
   \bottomrule
 \end{tabular}
 \end{adjustbox}
 \label{tab:meandiff}
\end{table*}

As shown in Table~\ref{tab:rop}, transfer learning demonstrates consistent performance improvements over direct learning across all levels of ROP data availability and under all evaluation metrics. Despite the relatively limited ROP dataset, transfer learning achieves and surpasses the critical threshold of 90 percent in every category. Remarkably, it requires only 973 images, i.e, less than 10 percent of the available ROP data,  to attain AUROC and Accuracy values exceeding 0.9. In particular, with limited ROP data, transfer learning achieves impressive scores of 0.97 for AUROC and 0.96 for Accuracy. By contrast, direct learning fails to reach the 0.9 threshold in any category, even when nearly the full ROP dataset is utilized. 
In the extreme data scarce  regime with less than $10$ percent of ROP data,  direct learning never reaches \(0.5\) under either Precision or Sensitivity.

Table \ref{tab:meandiff} further compares  the mean relative performance difference  between transfer learning and direct learning \(\Delta \text{metric}\). 
Here the advantage of transfer learning in the data-scarce regime is even more striking across the board. When the proportion of available ROP data is reduced to the extreme case of one percent (i.e., fewer than 100 images), transfer learning yields the most significant gains in both sensitivity and AUROC. These observations suggest that when the training data becomes increasingly limited, the benefit of transfer learning remains robust.  





\bibliography{ref,ref-add}

@inproceedings{agarwal2023rl,
  author    = {Agarwal, Alekh and Song, Yuda and Sun, Wen and Wang, Kaiwen and Wang, Mengdi and Zhang, Xuezhou},
  title     = {Provable Benefits of Representational Transfer in Reinforcement Learning},
  booktitle = {COLT},
  volume    = {36},
  pages     = {2114--2187},
  year      = {2023}
}

@article{cai2024bandits,
  author  = {Cai, Changxiao and Cai, T. Tony and Li, Hongzhe},
  title   = {Transfer Learning for Contextual Multi-Armed Bandits},
  journal = {Ann. Stat.},
  volume  = {52},
  number  = {1},
  pages   = {207--232},
  year    = {2024}
}

@article{cai2021nonparametric,
  author  = {Cai, T. Tony and Wei, Hongji},
  title   = {Transfer Learning for Nonparametric Classification: Minimax Rate and Adaptive Classifier},
  journal = {Ann. Stat.},
  volume  = {49},
  number  = {1},
  pages   = {100--128},
  year    = {2021}
}

@inproceedings{cheng2025diffusion,
  author    = {Cheng, Ziheng and Xie, Tianyu and Zhang, Shiyue and Zhang, Cheng},
  title     = {Provable Sample-Efficient Transfer Learning Conditional Diffusion Models via Representation Learning},
  booktitle = {Neurips},
  volume    = {38},
  year      = {2025}
}

@inproceedings{du2021fewshot,
  author    = {Du, Simon S. and Hu, Wei and Kakade, Sham M. and Lee, Jason D. and Lei, Qi},
  title     = {Few-Shot Learning via Learning the Representation, Provably},
  booktitle = {ICLR},
  year      = {2021}
}

@article{duan2023adaptive,
  author  = {Duan, Yaqi and Wang, Kaizheng},
  title   = {Adaptive and Robust Multi-Task Learning},
  journal = {Ann. Stat.},
  volume  = {51},
  number  = {5},
  pages   = {2015--2039},
  year    = {2023},
}

@article{fan2025robust,
  author  = {Fan, Jianqing and Gao, Cheng and Klusowski, Jason M.},
  title   = {Robust Transfer Learning with Unreliable Source Data},
  journal = {Ann. Stat.},
  volume  = {53},
  number  = {4},
  pages   = {1728--1752},
  year    = {2025},
}

@inproceedings{ge2024provable,
  author    = {Ge, Jiawei and Tang, Shange and Fan, Jianqing and Jin, Chi},
  title     = {On the Provable Advantage of Unsupervised Pretraining},
  booktitle = {ICLR},
  year      = {2024}
}

@inproceedings{hanneke2019value,
  author    = {Hanneke, Steve and Kpotufe, Samory},
  title     = {On the Value of Target Data in Transfer Learning},
  booktitle = {Neurips},
  volume    = {32},
  pages     = {9867--9877},
  year      = {2019}
}

@article{hanneke2022nofree,
  author  = {Hanneke, Steve and Kpotufe, Samory},
  title   = {A No-Free-Lunch Theorem for Multitask Learning},
  journal = {Ann. Stat.},
  volume  = {50},
  number  = {6},
  pages   = {3119--3143},
  year    = {2022}
}

@inproceedings{hanneke2023limits,
  author    = {Hanneke, Steve and Kpotufe, Samory and Mahdaviyeh, Yasaman},
  title     = {Limits of Model Selection under Transfer Learning},
  booktitle = {COLT},
  volume    = {36},
  pages     = {5781--5812},
  year      = {2023}
}

@inproceedings{he2024transfusion,
  author    = {He, Zelin and Sun, Ying and Li, Runze},
  title     = {{TransFusion}: Covariate-Shift Robust Transfer Learning for High-Dimensional Regression},
  booktitle = {AISTATS},
  volume    = {27},
  pages     = {703--711},
  year      = {2024}
}

@inproceedings{jonesmccormick2025provable,
  author    = {Jones-McCormick, Taj and Jagannath, Aukosh and Sen, Subhabrata},
  title     = {Provable Benefits of Unsupervised Pre-Training and Transfer Learning via Single-Index Models},
  booktitle = {ICLM},
  volume    = {42},
  pages     = {28350--28376},
  year      = {2025}
}

@article{li2022highdim,
  author  = {Li, Sai and Cai, T. Tony and Li, Hongzhe},
  title   = {Transfer Learning for High-Dimensional Linear Regression: Prediction, Estimation and Minimax Optimality},
  journal = {J. R. Stat. Soc. Ser. B Stat. Methodol.},
  volume  = {84},
  number  = {1},
  pages   = {149--173},
  year    = {2022}
}

@inproceedings{liu2024unified,
  author    = {Liu, Shuo Shuo},
  title     = {Unified Transfer Learning in High-Dimensional Linear Regression},
  booktitle = {AISTATS},
  volume    = {27},
  pages     = {1036--1044},
  year      = {2024}
}

@article{reeve2021adaptive,
  author  = {Reeve, Henry W. J. and Cannings, Timothy I. and Samworth, Richard J.},
  title   = {Adaptive Transfer Learning},
  journal = {Ann. Stat.},
  volume  = {49},
  number  = {6},
  pages   = {3618--3649},
  year    = {2021}
}

@inproceedings{tahir2024features,
  author    = {Tahir, Javan and Ganguli, Surya and Rotskoff, Grant M.},
  title     = {Features Are Fate: A Theory of Transfer Learning in High-Dimensional Regression},
  booktitle = {ICML},
  volume    = {42},
  pages     = {58142--58168},
  year      = {2025}
}

@article{tian2023glm,
  author  = {Tian, Ye and Feng, Yang},
  title   = {Transfer Learning under High-Dimensional Generalized Linear Models},
  journal = {J. Am. Stat. Assoc.},
  volume  = {118},
  number  = {544},
  pages   = {2684--2697},
  year    = {2023}
}

@article{tian2025similar,
  author  = {Tian, Ye and Gu, Yuqi and Feng, Yang},
  title   = {Learning from Similar Linear Representations: Adaptivity, Minimaxity, and Robustness},
  journal = {J. Mach. Learn. Res.},
  volume  = {26},
  number  = {187},
  pages   = {1--125},
  year    = {2025}
}

@inproceedings{tripuraneni2021provable,
  author    = {Tripuraneni, Nilesh and Jin, Chi and Jordan, Michael I.},
  title     = {Provable Meta-Learning of Linear Representations},
  booktitle = {ICML},
  volume    = {38},
  pages     = {10434--10443},
  year      = {2021}
}

@article {Stone82,
    AUTHOR = {Stone, Charles J.},
     TITLE = {Optimal global rates of convergence for nonparametric
              regression},
   JOURNAL = {Ann. Statist.},
    VOLUME = {10},
      YEAR = {1982},
    NUMBER = {4},
     PAGES = {1040--1053},
}

@article {Stone80,
    AUTHOR = {Stone, Charles J.},
     TITLE = {Optimal rates of convergence for nonparametric estimators},
   JOURNAL = {Ann. Statist.},
    VOLUME = {8},
      YEAR = {1980},
    NUMBER = {6},
     PAGES = {1348--1360},
}

@book {Villani03,
    AUTHOR = {Villani, C\'edric},
     TITLE = {Topics in optimal transportation},
    SERIES = {Graduate Studies in Mathematics},
    VOLUME = {58},
 PUBLISHER = {American Mathematical Society, Providence, RI},
      YEAR = {2003},
     PAGES = {xvi+370},
}

@book {Villani09,
    AUTHOR = {Villani, C\'edric},
     TITLE = {Optimal transport},
    SERIES = {Grundlehren der mathematischen Wissenschaften [Fundamental
              Principles of Mathematical Sciences]},
    VOLUME = {338},
      NOTE = {Old and new},
 PUBLISHER = {Springer-Verlag, Berlin},
      YEAR = {2009},
     PAGES = {xxii+973},
}

@article {Brenier87,
    AUTHOR = {Brenier, Yann},
     TITLE = {D\'ecomposition polaire et r\'earrangement monotone des champs
              de vecteurs},
   JOURNAL = {C. R. Acad. Sci. Paris S\'er. I Math.},
    VOLUME = {305},
      YEAR = {1987},
    NUMBER = {19},
     PAGES = {805--808},
}

@article {Brenier91,
    AUTHOR = {Brenier, Yann},
     TITLE = {Polar factorization and monotone rearrangement of
              vector-valued functions},
   JOURNAL = {Comm. Pure Appl. Math.},
    VOLUME = {44},
      YEAR = {1991},
    NUMBER = {4},
     PAGES = {375--417},
}

@article {Pass12,
    AUTHOR = {Pass, Brendan},
     TITLE = {Regularity of optimal transportation between spaces with
              different dimensions},
   JOURNAL = {Math. Res. Lett.},
    VOLUME = {19},
      YEAR = {2012},
    NUMBER = {2},
     PAGES = {291--307},
}

@article {GM96,
    AUTHOR = {Gangbo, Wilfrid and McCann, Robert J.},
     TITLE = {The geometry of optimal transportation},
   JOURNAL = {Acta Math.},
    VOLUME = {177},
      YEAR = {1996},
    NUMBER = {2},
     PAGES = {113--161},
}

@article {Ca92,
    AUTHOR = {Caffarelli, Luis A.},
     TITLE = {The regularity of mappings with a convex potential},
   JOURNAL = {J. Amer. Math. Soc.},
    VOLUME = {5},
      YEAR = {1992},
    NUMBER = {1},
     PAGES = {99--104},
}

@article {Ca922,
    AUTHOR = {Caffarelli, Luis A.},
     TITLE = {Boundary regularity of maps with convex potentials},
   JOURNAL = {Comm. Pure Appl. Math.},
    VOLUME = {45},
      YEAR = {1992},
    NUMBER = {9},
     PAGES = {1141--1151},
}

@article {CFJ17,
    AUTHOR = {Colombo, Maria and Figalli, Alessio and Jhaveri, Yash},
     TITLE = {Lipschitz changes of variables between perturbations of
              log-concave measures},
   JOURNAL = {Ann. Sc. Norm. Super. Pisa Cl. Sci. (5)},
    VOLUME = {17},
      YEAR = {2017},
    NUMBER = {4},
     PAGES = {1491--1519},
}

@article{CFS24,
  title={On optimal transport maps between 1/d-concave densities},
  author={Carlier, Guillaume and Figalli, Alessio and Santambrogio, Filippo},
journal={arXiv preprint arXiv:2404.05456},
  year={2024},
}

@article {Car00,
    AUTHOR = {Caffarelli, Luis A.},
     TITLE = {Monotonicity properties of optimal transportation and the
              {FKG} and related inequalities},
   JOURNAL = {Comm. Math. Phys.},
    VOLUME = {214},
      YEAR = {2000},
    NUMBER = {3},
     PAGES = {547--563},
}

@article {Ca96,
    AUTHOR = {Caffarelli, Luis A.},
     TITLE = {Boundary regularity of maps with convex potentials. {II}},
   JOURNAL = {Ann. of Math. (2)},
    VOLUME = {144},
      YEAR = {1996},
    NUMBER = {3},
     PAGES = {453--496},
}

@article {MP20,
    AUTHOR = {McCann, Robert J. and Pass, Brendan},
     TITLE = {Optimal transportation between unequal dimensions},
   JOURNAL = {Arch. Ration. Mech. Anal.},
    VOLUME = {238},
      YEAR = {2020},
    NUMBER = {3},
     PAGES = {1475--1520},
}

@article {MB24,
    AUTHOR = {Manole, Tudor and Balakrishnan, Sivaraman and Niles-Weed,
              Jonathan and Wasserman, Larry},
     TITLE = {Plugin estimation of smooth optimal transport maps},
   JOURNAL = {Ann. Statist.},
    VOLUME = {52},
      YEAR = {2024},
    NUMBER = {3},
     PAGES = {966--998},
}

@article{CNR24,
  title={Statistical optimal transport},
  author={Chewi, Sinho and Niles-Weed, Jonathan and Rigollet, Philippe},
  journal={arXiv preprint arXiv:2407.18163},
  year={2024},
}

@article{Cao23,
  title={Risk of transfer learning and its applications in finance},
  author={Cao, Haoyang and Gu, Haotian and Guo, Xin and Rosenbaum, Mathieu},
journal={arXiv preprint arXiv:2311.03283},
  year={2023},
}

@inproceedings{mahajan2025learning,
  title={Learning to Undo: Transfer Reinforcement Learning under Linear State Space Transformations},
  author={Mahajan, M. and Pacchiano, A. and Zhang, X.},
  booktitle={OpenReview},
  year={2025},
  url={https://openreview.net/forum?id=jI8a3s5xUz}
}

@article{ren2025partial,
  title={Partial Domain Adaptation via Importance Sampling-based Shift Correction},
  author={Ren, C. X. and Luo, Y. W. and Guo, C. J. and Xu, X. L.},
  journal={IEEE Trans. Neural Netw.},
  year={2025}
}

@article{xu2025ibex,
  title={IBEX: Information-Bottleneck-Explored Coarse-to-Fine Molecular Generation under Limited Data},
  author={Xu, D. and Yao, J. X. and Song, S. and Zhu, Z. and Ji, J.},
  journal={arXiv preprint arXiv:2508.10775},
  year={2025}
}

@article{apellaniz2025cancer,
  title={Advancing Cancer Research with Synthetic Data Generation in Low-Data Scenarios},
  author={Apellaniz, P. A. and Galende, B. A. and Jim, A.},
  journal={IEEE Journal of Biomedical Informatics},
  year={2025}
}

@article{bonneel2015sliced,
  title={Sliced and {R}adon {W}asserstein barycenters of measures},
  author={Bonneel, Nicolas and Rabin, Julien and Peyr{\'e}, Gabriel and Pfister, Hanspeter},
  journal={J. Math. Imaging Vision},
  volume={51},
  pages={22--45},
  year={2015}
}

@inproceedings{de2021diffusion,
  title={Diffusion schr{\"o}dinger bridge with applications to score-based generative modeling},
  author={De Bortoli, Valentin and Thornton, James and Heng, Jeremy and Doucet, Arnaud},
  booktitle={Neurips},
  volume={34},
  pages={17695--17709},
  year={2021},
}

@article{torous2021optimal,
  title={An Optimal Transport Approach to Estimating Causal Effects via Nonlinear Difference-in-Differences},
  author={Torous, William and Gunsilius, Florian and Rigollet, Philippe},
  journal={arXiv preprint arXiv:2108.05858},
  year={2021}
}

@article{courty2016optimal,
  title={Optimal transport for domain adaptation},
  author={Courty, Nicolas and Flamary, R{\'e}mi and Tuia, Devis and Rakotomamonjy, Alain},
  journal={IEEE Trans. Pattern Anal. Mach. Intell.},
  volume={39},
  number={9},
  pages={1853--1865},
  year={2016},
}

@inproceedings{eisenberg1990sample,
  title={On the sample complexity of pac-learning using random and chosen examples.},
  author={Eisenberg, Bonnie and Rivest, Ronald L},
  booktitle={COLT},
  volume={90},
  pages={154--162},
  year={1990}
}

@article{hanneke2016optimal,
  title={The optimal sample complexity of PAC learning},
  author={Hanneke, Steve},
  journal={J. Mach. Learn. Res.},
  volume={17},
  number={38},
  pages={1--15},
  year={2016}
}

@article{haussler1994bounds,
  title={Bounds on the sample complexity of Bayesian learning using information theory and the VC dimension},
  author={Haussler, David and Kearns, Michael and Schapire, Robert E},
  journal={Mach. Learn.},
  volume={14},
  number={1},
  pages={83--113},
  year={1994},
  publisher={Springer}
}

@inproceedings{ben2012hardness,
  title={On the hardness of domain adaptation and the utility of unlabeled target samples},
  author={Ben-David, Shai and Urner, Ruth},
  booktitle={ALT},
  pages={139--153},
  year={2012},
}

@inproceedings{devlin-etal-2019-bert,
    title = "{BERT}: Pre-training of Deep Bidirectional Transformers for Language Understanding",
    author = "Devlin, Jacob  and
      Chang, Ming-Wei  and
      Lee, Kenton  and
      Toutanova, Kristina",
    booktitle = "NACCL",
    volume={1},
    year = "2019",
    pages = "4171--4186",
}

@article{kim2022transfer,
  title={Transfer learning for medical image classification: A literature review},
  author={Kim, Hee E and Cosa-Linan, Alejandro and Santhanam, Nandhini and Jannesari, Mahboubeh and Maros, Mate E and Ganslandt, Thomas},
  journal={BMC Medical Imaging},
  volume={22},
  number={1},
  pages={69},
  year={2022},
  publisher={Springer}
}

@article{openai2023gpt4,
      title={G{P}{T}-4 {T}echnical {R}eport}, 
      author={OpenAI},
      journal = {arXiv preprint arXiv:2303.08774},
      year={2023},
}

@article{comanici2025gemini,
  title={Gemini 2.5: Pushing the frontier with advanced reasoning, multimodality, long context, and next generation agentic capabilities},
  author={Comanici, Gheorghe and Bieber, Eric and Schaekermann, Mike and Pasupat, Ice and Sachdeva, Noveen and Dhillon, Inderjit and Blistein, Marcel and Ram, Ori and Zhang, Dan and Rosen, Evan and others},
  journal={arXiv preprint arXiv:2507.06261},
  year={2025}
}

@article{grattafiori2024llama,
  title={The llama 3 herd of models},
  author={Grattafiori, Aaron and Dubey, Abhimanyu and Jauhri, Abhinav and Pandey, Abhinav and Kadian, Abhishek and Al-Dahle, Ahmad and Letman, Aiesha and Mathur, Akhil and Schelten, Alan and Vaughan, Alex and others},
  journal={arXiv preprint arXiv:2407.21783},
  year={2024}
}

@inproceedings{brunskill2013sample,
author = {Brunskill, Emma and Li, Lihong},
title = {Sample complexity of multi-task reinforcement learning},
year = {2013},
booktitle = {UAI},
volume = {29},
pages = {122–131},
}

@inproceedings{oguni2014reducing,
  title={Reducing Sample Complexity in Reinforcement Learning by Transferring Transition and Reward Probabilities},
  author={Oguni, Kouta and Narisawa, Kazuyuki and Shinohara, Ayumi},
  booktitle={ICAART},
  volume={2},
  pages={632--638},
  year={2014}
}

@inproceedings{tripuraneni2020theory,
  title={On the theory of transfer learning: The importance of task diversity},
  author={Tripuraneni, Nilesh and Jordan, Michael and Jin, Chi},
  booktitle={Neurips},
  volume={33},
  pages={7852--7862},
  year={2020}
}

@inproceedings{russo2023sample,
  title={On the sample complexity of representation learning in multi-task bandits with global and local structure},
  author={Russo, Alessio and Proutiere, Alexandre},
  booktitle={AAAI},
  volume={37},
  pages={9658--9667},
  year={2023}
}

@article{rigollet2025on,
author = {Philippe Rigollet and Austin J. Stromme},
title = {{On the sample complexity of entropic optimal transport}},
volume = {53},
journal = {Ann. Stat.},
number = {1},
pages = {61 -- 90},
year = {2025},
}

@article{vural2025unified,
  title={A Unified Analysis of Generalization and Sample Complexity for Semi-Supervised Domain Adaptation},
  author={Vural, Elif and Karaca, Huseyin},
  journal={arXiv preprint arXiv:2507.22632},
  year={2025}
}

@inproceedings{cuturi2013sinkhorn,
  title={Sinkhorn distances: Lightspeed computation of optimal transport},
  author={Cuturi, Marco},
  booktitle={NIPS},
  pages={2292–2300},
  volume={26},
  year={2013}
}

@article{zeng2019automatic,
  title={Automatic {I}{C}{D}-9 coding via deep transfer learning},
  author={Zeng, Min and Li, Min and Fei, Zhihui and Yu, Ying and Pan, Yi and Wang, Jianxin},
  journal={Neurocomputing},
  volume={324},
  pages={43--50},
  year={2019},
  publisher={Elsevier}
}

@article{wang2022transfer,
  title={Transfer Learning for Retinal Vascular Disease Detection: A Pilot Study with Diabetic Retinopathy and Retinopathy of Prematurity},
  author={Wang, Guan and Kikuchi, Yusuke and Yi, Jinglin and Zou, Qiong and Zhou, Rui and Guo, Xin},
  journal={arXiv preprint arXiv:2201.01250},
  year={2022}
}

@inproceedings{ruder2019transfer,
  title={Transfer learning in natural language processing},
  author={Ruder, Sebastian and Peters, Matthew E and Swayamdipta, Swabha and Wolf, Thomas},
  booktitle={NACCL},
  pages={15--18},
  year={2019}
}

@inproceedings{sung2022vl,
  title={Vl-adapter: Parameter-efficient transfer learning for vision-and-language tasks},
  author={Sung, Yi-Lin and Cho, Jaemin and Bansal, Mohit},
  booktitle={CVPR},
  pages = {5227--5237},
  year={2022}
}

@article{banerjee2005optimality,
  title={On the optimality of conditional expectation as a Bregman predictor},
  author={Banerjee, Arindam and Guo, Xin and Wang, Hui},
  journal={IEEE Trans. Inf. Theory},
  volume={51},
  number={7},
  pages={2664--2669},
  year={2005},
}

@inproceedings{saenko2010adapting,
  title={Adapting visual category models to new domains},
  author={Saenko, Kate and Kulis, Brian and Fritz, Mario and Darrell, Trevor},
  booktitle={Proceedings of the 11th European Conference on Computer Vision},
  pages={213--226},
  year={2010},
  organization={Springer}
}

@misc{chollet2015keras,
  title={Keras},
  author={Chollet, Fran\c{c}ois and others},
  year={2015},
  howpublished={https://keras.io},
}

@article {Fiersone20183061,
	author = {Fierson, Walter M.},
	title = {Screening Examination of Premature Infants for Retinopathy of Prematurity},
	volume = {142},
	number = {6},
	elocation-id = {e20183061},
	year = {2018},
	journal = {Pediatrics}
}

@article{Blencowe2013,
author = {Blencowe, Hannah and Lawn, Joy and Vazquez, Thomas and Fielder, Alistair and Gilbert, Clare},
year = {2013},
month = {12},
pages = {35-49},
title = {Preterm-associated visual impairment and estimates of retinopathy of prematurity at regional and global levels for 2010},
volume = {74 Suppl 1},
journal = {Pediatr. Res.},
}

@article{gilbert_rahi_eckstein_osullivan_foster_1997, 
title={Retinopathy of prematurity in middle-income countries}, 
volume={350}, 
number={9070}, 
journal={The Lancet}, 
author={Gilbert, Clare and Rahi, Jugnoo and Eckstein, Michael and O'sullivan, Jane and Foster, Allen}, 
year={1997}, 
pages={12–14}}

@article{10.1001/archopht.121.12.1684,
    author = {{Early Treatment For Retinopathy Of Prematurity Cooperative Group}},
    title = "{Revised Indications for the Treatment of Retinopathy of Prematurity: Results of the Early Treatment for Retinopathy of Prematurity Randomized Trial}",
    journal = {Archives of Ophthalmology},
    volume = {121},
    number = {12},
    pages = {1684-1694},
    year = {2003}
}

@article{permuter2006study,
  title={A study of Gaussian mixture models of color and texture features for image classification and segmentation},
  author={Permuter, Haim and Francos, Joseph and Jermyn, Ian},
  journal={Pattern recognition},
  volume={39},
  number={4},
  pages={695--706},
  year={2006},
  publisher={Elsevier}
}

\bibliographystyle{IEEEtran}

\end{document}